\documentclass{article}
\usepackage[final,nonatbib]{neurips_2024}

\usepackage[utf8]{inputenc} 
\usepackage[T1]{fontenc}    
\usepackage{url}            
\usepackage{booktabs}       
\usepackage{amsfonts}       
\usepackage{nicefrac}       
\usepackage{microtype}      

\usepackage{algorithm}
\usepackage{algorithmicx}
\usepackage[noend]{algpseudocode}
\usepackage{amsmath}
\usepackage{amssymb}
\usepackage{amsthm}
\usepackage{bbm}
\usepackage{bm}
\usepackage{color}
\usepackage{dirtytalk}
\usepackage{dsfont}
\usepackage{enumerate}
\usepackage{graphicx}
\usepackage{listings}
\usepackage{mathtools}
\usepackage[numbers]{natbib}
\usepackage{subfigure}
\usepackage{times}
\usepackage{xspace}

\usepackage[dvipsnames]{xcolor}
\usepackage[bookmarks=false]{hyperref}
\hypersetup{
  pdffitwindow=true,
  pdfstartview={FitH},
  pdfnewwindow=true,
  colorlinks,
  linktocpage=true,
  urlcolor=Green,
  linkcolor=Green,
  citecolor=Green
}
\usepackage[capitalize,noabbrev]{cleveref}

\usepackage[textsize=tiny]{todonotes}

\usepackage{thmtools}
\declaretheorem[name=Theorem,refname={Theorem,Theorems},Refname={Theorem,Theorems}]{theorem}
\declaretheorem[name=Lemma,refname={Lemma,Lemmas},Refname={Lemma,Lemmas},sibling=theorem]{lemma}

\declaretheorem[name=Assumption,refname={Assumption,Assumptions},Refname={Assumption,Assumptions}]{assumption}

\newcommand{\cA}{\mathcal{A}}

\newcommand{\cD}{\mathcal{D}}

\newcommand{\cN}{\mathcal{N}}

\newcommand{\cX}{\mathcal{X}}

\newcommand{\realset}{\mathbb{R}}

\newcommand{\E}[1]{\mathbb{E}\left[#1\right]}

\newcommand{\Erv}[2]{\mathbb{E}_{#1}\left[#2\right]}

\newcommand{\prob}[1]{\mathbb{P}\left(#1\right)}

\newcommand{\var}[1]{\mathrm{var}\left[#1\right]}

\newcommand{\cov}[1]{\mathrm{cov}\left[#1\right]}

\newcommand*\dif{\mathop{}\!\mathrm{d}}

\newcommand{\I}[1]{\mathds{1} \! \left\{#1\right\}}

\newcommand{\normw}[2]{\|#1\|_{#2}}

\newcommand{\set}[1]{\left\{#1\right\}}

\newcommand{\T}{^\top}

\DeclareMathOperator*{\argmax}{arg\,max\,}
\DeclareMathOperator*{\argmin}{arg\,min\,}
\let\det\relax
\DeclareMathOperator{\det}{det}

\let\trace\relax
\DeclareMathOperator{\trace}{tr}
\mathchardef\mhyphen="2D

\newcommand{\diffts}{\ensuremath{\color{Green}\tt DiffTS}\xspace}
\newcommand{\dps}{\ensuremath{\color{Green}\tt DPS}\xspace}
\newcommand{\irls}{\ensuremath{\color{Green}\tt IRLS}\xspace}
\newcommand{\laplacedps}{\ensuremath{\color{Green}\tt LaplaceDPS}\xspace}
\newcommand{\linucb}{\ensuremath{\color{Green}\tt LinUCB}\xspace}
\newcommand{\mixts}{\ensuremath{\color{Green}\tt MixTS}\xspace}
\newcommand{\ts}{\ensuremath{\color{Green}\tt TS}\xspace}
\newcommand{\tunedts}{\ensuremath{\color{Green}\tt TunedTS}\xspace}

\title{Online Posterior Sampling with a Diffusion Prior}

\author{
  Branislav Kveton \\
  Adobe Research\thanks{The work was done at AWS AI Labs.}
  \And
  Boris N. Oreshkin \\
  Amazon
  \And
  Youngsuk Park \\
  AWS AI Labs
  \And
  Aniket Deshmukh \\
  AWS AI Labs
  \And
  Rui Song \\
  Amazon
}

\begin{document}

\maketitle

\begin{abstract}
Posterior sampling in contextual bandits with a Gaussian prior can be implemented exactly or approximately using the Laplace approximation. The Gaussian prior is computationally efficient but it cannot describe complex distributions. In this work, we propose approximate posterior sampling algorithms for contextual bandits with a diffusion model prior. The key idea is to sample from a chain of approximate conditional posteriors, one for each stage of the reverse diffusion process, which are obtained by the Laplace approximation. Our approximations are motivated by posterior sampling with a Gaussian prior, and inherit its simplicity and efficiency. They are asymptotically consistent and perform well empirically on a variety of contextual bandit problems.
\end{abstract}

\section{Introduction}
\label{sec:introduction}

A \emph{multi-armed bandit} \citep{lai85asymptotically,auer02finitetime,lattimore19bandit} is an online learning problem where an agent sequentially interacts with an environment over $n$ rounds with the goal of maximizing its rewards. In each round, it takes an \emph{action} and receives its \emph{stochastic reward}. The mean rewards of the actions are unknown \emph{a priori} and must be learned. This leads to the \emph{exploration-exploitation dilemma}: \emph{explore} actions to learn about them or \emph{exploit} the action with the highest estimated reward. Bandits have been successfully applied to problems where uncertainty modeling and adaptation are beneficial, such recommender systems \citep{li10contextual,zhao13interactive,kawale15efficient,li16collaborative} and hyper-parameter optimization \citep{li18hyperband}.

Contextual bandits \citep{langford08epochgreedy,li10contextual} with linear \citep{dani08stochastic,abbasi-yadkori11improved} and \emph{generalized linear models (GLMs)} \citep{filippi10parametric,li17provably,abeille17linear,kveton20randomized} have become popular due to the their flexibility and efficiency. The features in these models can be hand-crafted or learned from historic data \citep{riquelme18deep}, and the models can be also updated incrementally \citep{abbasi-yadkori11improved,jun17scalable}. While the original algorithms for linear and GLM bandits were based on \emph{upper confidence bounds (UCBs)} \citep{dani08stochastic,abbasi-yadkori11improved,filippi10parametric}, \emph{Thompson sampling (TS)} is more popular in practice \citep{chapelle11empirical,agrawal13thompson,russo14learning,russo18tutorial}. The key idea in TS is to explore by sampling from the posterior distribution of model parameter $\theta_*$. TS uses the prior knowledge about $\theta_*$ to speed up exploration \citep{chapelle11empirical,riquelme18deep,lu19informationtheoretic,basu21noregrets,hong22hierarchical,hong22deep,aouali23mixedeffect}. When the prior is a multivariate Gaussian, the posterior of $\theta_*$ can be updated and sampled from efficiently \citep{chapelle11empirical}. This prior has a limited expressive power, because it cannot even represent multimodal distributions. To address this, we study posterior sampling with a diffusion prior. The main benefit of such priors is that they can represent complex distributions and be learned from data.

We make the following contributions. First, we propose novel posterior sampling approximations for linear models and GLMs with a diffusion model prior. The key idea is to sample from a chain of approximate conditional posteriors, one for each stage of the reverse process, which are estimated in a closed form. In linear models, each conditional is a product of two Gaussians, representing prior knowledge and diffused evidence (\cref{thm:linear posterior}). In GLMs, each conditional is obtained by a Laplace approximation, which mixes prior knowledge and evidence (\cref{thm:glm posterior}). Our approximations are motivated by posterior sampling with Gaussian priors, and inherit its simplicity and efficiency. Prior works (\cref{sec:related work}) sampled from the posterior using the likelihood score, and their approximations become unstable when the score is high. We combine the likelihood with conditional priors, in each stage of the diffusion model, using the Laplace approximation. The resulting posterior concentrates at a single point and can be sampled from efficiently even if the likelihood score is high. We prove that this approximation is asymptotically consistent.

Our second contribution is in theory. We properly derive our posterior approximations (\cref{thm:linear posterior,thm:glm posterior}) and show their asymptotic consistency (\cref{thm:asymptotic consistency}). The key idea in the proof of \cref{thm:asymptotic consistency} is that the conditional posteriors concentrate at a scaled unknown model parameter as the number of observations increases. While this claim is asymptotic, it is an expected property of a posterior distribution. Many prior works, such as \citet{chung23diffusion}, do not propose asymptotically consistent approximations. All of our main results rely on a novel approximation of clean samples by scaled diffused samples (\cref{sec:key approximation}). The most challenging part of the analysis is \cref{thm:asymptotic consistency}, where we analyze an asymptotic behavior of a chain of $T$ dependent random vectors.

Our last contribution is an empirical evaluation on contextual bandits. We focus on bandits because the ability to represent all levels of uncertainty precisely is critical for exploration. Our experiments show that a score-based method fails to do so (\cref{sec:synthetic experiment}). Note that our posterior approximations are general and not restricted to bandits.

\section{Setting}
\label{sec:setting}

We start with introducing our notation. Random variables are capitalized, except for Greek letters like $\theta$. We denote the marginal and conditional probabilities under probability measure $p$ by $p(X = x)$ and $p(X = x \mid Y = y)$, respectively. When the random variables are clear from context, we write $p(x)$ and $p(x \mid y)$. We denote by $X_{n : m}$ and $x_{n : m}$ a collection of random variables and their values, respectively. For a positive integer $n$, we define $[n] = \set{1, \dots, n}$. The indicator function is $\I{\cdot}$. The $i$-th entry of vector $v$ is $v_i$. If the vector is already indexed, such as $v_j$, we write $v_{j, i}$. We denote the maximum and minimum eigenvalues of matrix $M \in \realset^{d \times d}$ by $\lambda_1(M)$ and $\lambda_d(M)$, respectively.

The posterior sampling problem can be formalized as follows. Let $\theta_* \in \Theta$ be an unknown \emph{model parameter} and $\Theta \subseteq \realset^d$ be the space of model parameters. Let $h = \set{(\phi_\ell, y_\ell)}_{\ell \in [N]}$ be the \emph{history} of $N$ noisy observations of $\theta_*$, where $\phi_\ell \in \realset^d$ is the feature vector for $y_\ell \in \realset$. We assume that
\begin{align}
  y_\ell
  = g(\phi_\ell\T \theta_*) + \varepsilon_\ell\,,
  \label{eq:observation model}
\end{align}
where $g: \realset \to \realset$ is the \emph{mean function} and $\varepsilon_\ell$ is an independent zero-mean $\sigma^2$-sub-Gaussian noise for $\sigma > 0$. Let $p(h \mid \theta_*)$ be the \emph{likelihood} of observations in history $h$ under model parameter $\theta_*$ and $p(\theta_*)$ be its \emph{prior probability}. By Bayes' rule, the posterior distribution of $\theta_*$ given $h$ is
\begin{align}
  p(\theta_* \mid h)
  \propto p(h \mid \theta_*) \, p(\theta_*)\,.
  \label{eq:posterior distribution}
\end{align}
We want to sample from $p(\cdot \mid h)$ efficiently when the prior distribution is represented by a diffusion model. As a stepping stone, we review existing posterior formulas for multivariate Gaussian priors. This motivates our solution for diffusion model priors.

\subsection{Linear Model}
\label{sec:linear model}

The posterior of $\theta_*$ in linear models can be derived as follows.

\begin{assumption}
\label{ass:linear model likelihood} Let $g$ in \eqref{eq:observation model} be an identity and $\varepsilon_\ell \sim \cN(0, \sigma^2)$. Then the likelihood of $h$ under model parameter $\theta_*$ is $p(h \mid \theta_*) \propto \exp[- \sum_{\ell = 1}^N (y_\ell - \phi_\ell\T \theta_*)^2 / (2 \sigma^2)]$.
\end{assumption}

Let $p(\theta_*) = \cN(\theta_*; \theta_0, \Sigma_0)$ be the prior distribution of $\theta_*$, where $\theta_0 \in \realset^d$ and $\Sigma_0 \in \realset^{d \times d}$ are the prior mean and covariance, respectively. Then $p(\theta_* \mid h) \propto \cN(\theta_*; \hat{\theta}, \hat{\Sigma})$ \citep{bishop06pattern}, where
\begin{align*}
  \textstyle
  \hat{\theta}
  = \hat{\Sigma} \left(\Sigma_0^{-1} \theta_0 +
  \sigma^{-2} \sum_{\ell = 1}^N \phi_\ell y_\ell\right)\,, \quad
  \hat{\Sigma}
  = \left(\Sigma_0^{-1} +
  \sigma^{-2} \sum_{\ell = 1}^N \phi_\ell \phi_\ell\T\right)^{-1}\,,
\end{align*}
are the posterior mean and covariance, respectively. In this work, we write them equivalently as
\begin{align}
  \hat{\theta}
  = \hat{\Sigma} (\Sigma_0^{-1} \theta_0 + \bar{\Sigma}^{-1} \bar{\theta})\,, \quad
  \hat{\Sigma}
  = (\Sigma_0^{-1} + \bar{\Sigma}^{-1})^{-1}\,,
  \label{eq:linear posterior}
\end{align}
where $\bar{\theta} = \sigma^{-2} \bar{\Sigma} \sum_{\ell = 1}^N \phi_\ell y_\ell$ and $\bar{\Sigma}^{-1} = \sigma^{-2} \sum_{\ell = 1}^N \phi_\ell \phi_\ell\T$ are the empirical mean and inverse of its covariance, respectively. Therefore, the posterior of $\theta_*$ is a product of two multivariate Gaussians: $\cN(\theta_0, \Sigma_0)$ representing prior knowledge about $\theta_*$ and $\cN(\bar{\theta}, \bar{\Sigma})$ representing empirical evidence.

\subsection{Generalized Linear Model}
\label{sec:glm}

\begin{algorithm}[t]
  \caption{\irls: Iteratively reweighted least squares.}
  \label{alg:irls}
  \begin{algorithmic}[1]
    \State \textbf{Input:} Prior parameters $\theta_0$ and $\Sigma_0$, history of observations $h = \set{(\phi_\ell, y_\ell)}_{\ell \in [N]}$
    \Statex \vspace{-0.05in}
    \State Initialize $\hat{\theta} \in \realset^d$
    \Repeat
      \For{stage $\ell = 1, \dots, N$}
        \State $z_\ell \gets \phi_\ell\T \hat{\theta} +
        (y_\ell - g(\phi_\ell\T \hat{\theta})) / \dot{g}(\phi_\ell\T \hat{\theta})$
      \EndFor
      \State $\hat{\Sigma} \gets \left(\Sigma_0^{-1} + \sum_{\ell = 1}^N
      \dot{g}(\phi_\ell\T \hat{\theta}) \phi_\ell \phi_\ell\T\right)^{-1}$
      \State $\hat{\theta} \gets \hat{\Sigma} \left(\Sigma_0^{-1} \theta_0 +
      \sum_{\ell = 1}^N \dot{g}(\phi_\ell\T \hat{\theta}) \phi_\ell z_\ell\right)$
    \Until{$\hat{\theta}$ converges}
    \Statex \vspace{-0.05in}
    \State \textbf{Output:} Posterior mean $\hat{\theta}$ and covariance $\hat{\Sigma}$
  \end{algorithmic}
\end{algorithm}

\emph{Generalized linear models (GLMs)} \citep{mccullagh89generalized} extend linear models (\cref{sec:linear model}) to non-linear monotone \emph{mean functions} $g$ in \eqref{eq:observation model}. For instance, in logistic regression, $g(u) = 1 / (1 + \exp[- u])$ is a sigmoid. The likelihood of observations in GLMs has the following form \citep{kveton20randomized}.

\begin{assumption}
\label{ass:glm likelihood} Let $h = \set{(\phi_\ell, y_\ell)}_{\ell \in [N]}$ be a history of $N$ observations under mean function $g$ and the corresponding noise. Then $\log p(h \mid \theta_*) \propto \sum_{\ell = 1}^N y_\ell \phi_\ell\T \theta_* - b(\phi_\ell\T \theta_*) + c(y_\ell)$, where $c$ is a real function and $b$ is a function whose derivative is the mean function, $\dot{b} = g$.
\end{assumption}

The posterior distribution of $\theta_*$ in GLMs does not have a closed form in general \citep{bishop06pattern}. Therefore, it is often approximated by the \emph{Laplace approximation}. Let the prior distribution of the model parameter be $p(\theta_*) = \cN(\theta_*; \theta_0, \Sigma_0)$, as in \cref{sec:linear model}. Then the Laplace approximation is $\cN(\hat{\theta}, \hat{\Sigma})$, where $\hat{\theta}$ is the \emph{maximum a posteriori (MAP) estimate} of $\theta_*$ and $\hat{\Sigma}$ is the corresponding covariance. Note that the Laplace approximation can be applied to non-Gaussian priors.

The MAP estimate $\hat{\theta}$ can be obtained by \emph{iteratively reweighted least squares (IRLS)} \citep{wolke88iteratively}, which we present in \cref{alg:irls}. \irls is a Newton-type algorithm that computes $\hat{\theta}$ iteratively (lines 6 and 7). It converges to the optimal solution due to the strong convexity of the problem. The solution has a similar structure to \eqref{eq:linear posterior}. That is, $\cN(\hat{\theta}, \hat{\Sigma})$ is a product of two multivariate Gaussians, representing prior knowledge about $\theta_*$ and empirical evidence. The new quantities in GLMs are the derivative of the mean function $\dot{g}$ and pseudo-observations $z_\ell$ (line 5), which play the role of observations $y_\ell$ in \cref{sec:linear model}.

\subsection{Towards Diffusion Model Priors}
\label{sec:towards diffusion model priors}

The assumption that $p(\theta_*) = \cN(\theta_*; \theta_0, \Sigma_0)$ is limiting, for instance because it precludes multimodal priors. We relax it by representing $p(\theta_*)$ by a diffusion model, which we call a \emph{diffusion model prior}. We propose efficient posterior sampling approximations for this prior, where the prior and empirical evidence are mixed similarly to \eqref{eq:linear posterior} and \irls. We review diffusion models next.

\section{Diffusion Models}
\label{sec:diffusion models}

Diffusion models \citep{sohldickstein15deep,ho20denoising} are generative models trained by diffusing samples from unknown and hard to represent distributions. They can be viewed in multiple ways \citep{song21scorebased}. We adopt the probabilistic formulation and presentation of \citet{ho20denoising}. A \emph{diffusion model} is a graphical model with $T$ stages indexed by $t \in [T]$. Each stage $t$ is associated with a \emph{latent variable} $S_t \in \realset^d$. A \emph{sample} from the model is represented by an \emph{observed variable} $S_0 \in \realset^d$. We visualize a diffusion model in \cref{fig:diffusion model}. In the \emph{forward process}, a clean sample $s_0$ is diffused through a sequence of variables $S_1, \dots, S_T$. This process is used to learn the \emph{reverse process}, where the clean sample $s_0$ is generated through a sequence of variables $S_T, \dots, S_0$. To sample $s_0$ from the posterior (\cref{sec:posterior sampling}), we add a random variable $H$ that represents partial information about $s_0$. We introduce forward and reverse diffusion processes next. Learning of the reverse process is described in \cref{sec:learning reverse process}. While this is a critical component of diffusion models, it is not necessary to introduce our posterior approximations.

\textbf{Forward process.} In the forward process, a clean sample $s_0$ is diffused through a chain of latent variables $S_1, \dots S_T$ (\cref{fig:diffusion model}). We denote the probability measure under this process by $q$ and define its joint probability distribution as
\begin{align}
  \textstyle
  q(s_{1 : T} \mid s_0)
  = \prod_{t = 1}^T q(s_t \mid s_{t - 1})\,, \quad
  \forall t \in [T]:
  q(s_t \mid s_{t - 1})
  & = \cN(s_t; \sqrt{\alpha_t} s_{t - 1}, \beta_t I_d)\,,
  \label{eq:forward process}
\end{align}
where $q(s_t \mid s_{t - 1})$ is the conditional density of mapping a less diffused $s_{t - 1}$ to a more diffused $s_t$. The diffusion rate is set by parameters $\alpha_t \in (0, 1)$ and $\beta_t = 1 - \alpha_t$. The forward process is sampled from as follows. First, a clean sample $s_0$ is chosen. Then $S_t \sim q(\cdot \mid s_{t - 1})$ are sampled, from $t = 1$ to $t = T$.

\textbf{Reverse process.} In the reverse process, a clean sample $s_0$ is generated through a chain of variables $S_T, \dots, S_0$ (\cref{fig:diffusion model}). We denote the probability measure under this process by $p$ and define its joint probability distribution as
\begin{align}
  \textstyle
  p(s_{0 : T})
  & = p(s_T) \prod_{t = 1}^T p(s_{t - 1} \mid s_t)\,,
  \label{eq:reverse process} \\
  p(s_T) &
  = \cN(s_T; \mathbf{0}_d, I_d)\,, \quad
  \forall t \in [T]:
  p(s_{t - 1} \mid s_t)
  = \cN(s_{t - 1}; \mu_t(s_t), \Sigma_t)\,,
  \nonumber
\end{align}
where $p(s_{t - 1} \mid s_t)$ is the conditional density of mapping a more diffused $s_t$ to a less diffused $s_{t - 1}$. The function $\mu_t$ predicts the mean of $S_{t - 1} \mid s_t$ and is learned (\cref{sec:learning reverse process}). As in \citet{ho20denoising}, we keep the covariance fixed at $\Sigma_t = \tilde{\beta}_t I_d$, where $\tilde{\beta}_t = \frac{1 - \bar{\alpha}_{t - 1}}{1 - \bar{\alpha}_t} \beta_t$ and $\bar{\alpha}_t = \prod_{\ell = 1}^t \alpha_\ell$. This is known as a \emph{stable diffusion}. We make this assumption only to simplify exposition. All our derivations in \cref{sec:posterior sampling} hold when $\Sigma_t$ is learned, for instance as in \citet{bao22estimating}.

This process is called reverse because it is learned by reversing the forward process. The reverse process is sampled from as follows. First, an initial diffused sample $S_T \sim p$ is sampled. After that, $S_{t - 1} \sim p(\cdot \mid s_t)$ are sampled, from $t = T$ to $t = 1$.

\begin{figure}[t]
  \centering
  \begin{tabular}{c c}
    Forward process (probability measure $q$) &
    Reverse process (probability measure $p$) \\
    $S_T \gets S_{T - 1} \gets \dots \gets S_1 \gets S_0$ &
    $S_T \to S_{T - 1} \to \dots \to S_1 \to S_0 \to H$
  \end{tabular}
  \caption{Graphical models of the forward and reverse processes in the diffusion model. The variable $H$ represents partial information about $S_0$.}
  \label{fig:diffusion model}
\end{figure}

\section{Posterior Sampling}
\label{sec:posterior sampling}

This section is organized as follows. In \cref{sec:chain model posterior}, we show how to sample from a chain of random variables conditioned on observations. In \cref{sec:linear posterior,sec:glm posterior}, we specialize this to the observation models in \cref{sec:setting}.

\subsection{Chain Model Posterior}
\label{sec:chain model posterior}

Let $h = \set{(\phi_\ell, y_\ell)}_{\ell \in [N]}$ denote a \emph{history} of $N$ observations (\cref{sec:setting}) and $H$ be the corresponding random variable. In this section, we assume that $h$ is fixed. The Markovian structure of the reverse process (\cref{fig:diffusion model}) implies that the joint probability distribution conditioned on $h$ factors as
\begin{align*}
  p(s_{0 : T} \mid h)
  = p(s_T \mid h) \prod_{t = 1}^T p(s_{t - 1} \mid s_t, h)\,.
\end{align*}
Therefore, $p(s_{0 : T} \mid h)$ can be sampled from efficiently by first sampling from $p(s_T \mid h)$ and then from $T$ conditional distributions $p(s_{t - 1} \mid s_t, h)$. We derive these next.

\begin{lemma}
\label{lem:chain model posterior} Let $p$ be a probability measure over the reverse process (\cref{fig:diffusion model}). Then
\begin{align*}
  p(s_T \mid h)
  & \propto \textstyle \int_{s_0} p(h \mid s_0) \, p(s_0 \mid s_T) \dif s_0 \, p(s_T)\,, \\
  \forall t \in [T] \setminus \set{1}: p(s_{t - 1} \mid s_t, h)
  & \propto \textstyle \int_{s_0} p(h \mid s_0) \, p(s_0 \mid s_{t - 1}) \dif s_0 \,
  p(s_{t - 1} \mid s_t)\,, \\
  p(s_0 \mid s_1, h)
  & \propto p(h \mid s_0) \, p(s_0 \mid s_1)\,.
\end{align*}
\end{lemma}
\begin{proof}
The claim is proved in \cref{sec:chain model posterior proof}.
\end{proof}

\subsection{Linear Model Posterior}
\label{sec:linear posterior}

Now we specialize \cref{lem:chain model posterior} to the diffusion model prior (\cref{sec:diffusion models}) and linear models (\cref{sec:linear model}). The prior distribution is the reverse process in \eqref{eq:reverse process},
\begin{align*}
  p(s_T)
  = \cN(s_T; \mathbf{0}_d, I_d)\,, \quad
  \forall t \in [T]:
  p(s_{t - 1} \mid s_t)
  = \cN(s_{t - 1}; \mu_t(s_t), \Sigma_t)\,.
\end{align*}
The term $p(h \mid s_0)$ is the likelihood of observations in \cref{ass:linear model likelihood}. The main challenge in using the lemma is that the conditional densities of clean samples $p(s_0 \mid S_T)$ and $p(s_0 \mid s_t)$ are complex \citep{chung23diffusion}. To get around this, we make an additional assumption, which is discussed in \cref{sec:key approximation}.

\begin{theorem}
\label{thm:linear posterior} Let $p$ be a probability measure over the reverse process (\cref{fig:diffusion model}). Let $\bar{\theta}$ and $\bar{\Sigma}^{-1}$ be defined as in \eqref{eq:linear posterior}. Suppose that
\begin{align}
  \textstyle
  \int_{s_0} p(h \mid s_0) \, p(s_0 \mid s_t) \dif s_0
  \propto p(h \mid s_t / \sqrt{\bar{\alpha}_t})
  \label{eq:integral approximation}
\end{align}
holds for all $t \in [T]$. Then $p(s_T \mid h) \propto \cN(s_T; \hat{\mu}_{T + 1}(h), \hat{\Sigma}_{T + 1}(h))$, where
\begin{align}
  \hat{\mu}_{T + 1}(h)
  = \hat{\Sigma}_{T + 1}(h)
  (\underbrace{I_d \, \mathbf{0}_d}_{\emph{Prior}} +
  \underbrace{\bar{\Sigma}^{-1}
  \bar{\theta} / \sqrt{\bar{\alpha}_T}}_{\emph{Evidence}})\,, \quad
  \hat{\Sigma}_{T + 1}(h)
  = (\underbrace{I_d}_{\emph{Prior}} +
  \underbrace{\bar{\Sigma}^{-1} / \bar{\alpha}_T}_{\emph{Evidence}})^{-1}\,.
  \label{eq:posterior T}
\end{align}
For $t \in [T]$, we have $p(s_{t - 1} \mid s_t, h) \propto \cN(s_{t - 1}; \hat{\mu}_t(s_t, h), \hat{\Sigma}_t(h))$, where
\begin{align}
  \hat{\mu}_t(s_t, h)
  = \hat{\Sigma}_t(h)
  (\underbrace{\Sigma_t^{-1} \mu_t(s_t)}_{\emph{Prior}} +
  \underbrace{\bar{\Sigma}^{-1}
  \bar{\theta} / \sqrt{\bar{\alpha}_{t - 1}}}_{\emph{Evidence}})\,, \quad
  \hat{\Sigma}_t(h)
  = (\underbrace{\Sigma_t^{-1}}_{\emph{Prior}} +
  \underbrace{\bar{\Sigma}^{-1} / \bar{\alpha}_{t - 1}}_{\emph{Evidence}})^{-1}\,.
  \label{eq:posterior t - 1}
\end{align}
\end{theorem}
\begin{proof}
The proof is in \cref{sec:linear posterior proof}. It has four steps. First, we fix stage $t$ and apply approximation \eqref{eq:integral approximation}. Second, we rewrite the likelihood as in \eqref{eq:linear posterior}. Third, we reparameterize it as a function of $s_t$. At the end, we combine the likelihood with the Gaussian prior using \cref{lem:gaussian product} in \cref{sec:supporting lemmas}.
\end{proof}

\begin{algorithm}[t]
  \caption{\laplacedps: Laplace posterior sampling with a diffusion model prior.}
  \label{alg:diffusion posterior sampling}
  \begin{algorithmic}[1]
    \State \textbf{Input:} Diffusion model parameters $(\mu_t, \Sigma_t)_{t \in [T]}$, history of observations $h$
    \Statex \vspace{-0.05in}
    \State Initial sample $S_T \sim \cN(\hat{\mu}_{T + 1}(h), \hat{\Sigma}_{T + 1}(h))$
    \For{stage $t = T, \dots, 1$}
      \State $S_{t - 1} \sim \cN(\hat{\mu}_t(S_t, h), \hat{\Sigma}_t(h))$
    \EndFor
    \Statex \vspace{-0.1in}
    \State \textbf{Output:} Posterior sample $S_0$
  \end{algorithmic}
\end{algorithm}

The algorithm that samples from the posterior distribution in \cref{thm:linear posterior} is presented in \cref{alg:diffusion posterior sampling}. We call it \emph{Laplace diffusion posterior sampling (\laplacedps)} because its generalization to GLMs uses the Laplace approximation. \laplacedps samples from a chain of products of two distributions: one distribution represents the pre-trained diffusion model and does not depend on history $h$, and the other represents the history $h$. The sampling is implemented as follows. The initial variable $S_T$ is sampled conditioned on $h$ (line 2) from the distribution in \eqref{eq:posterior T}. This distribution is a product of the $h$-independent prior $\cN(\mathbf{0}_d, I_d)$ and the $h$-dependent distribution of the diffused evidence up to stage $T$, $\cN(\sqrt{\bar{\alpha}_T} \bar{\theta}, \bar{\alpha}_T \bar{\Sigma})$. Then, for any $t \in [T]$, $S_{t - 1}$ is sampled conditioned on $s_t$ and evidence $h$ (line 4) from the distribution in \eqref{eq:posterior t - 1}. This distribution is a product of the $h$-independent conditional prior $\cN(\mu_t(s_t), \Sigma_t)$, from the pre-trained model, and the $h$-dependent distribution of the diffused evidence up to stage $t - 1$, $\cN(\sqrt{\bar{\alpha}_{t - 1}} \bar{\theta}, \bar{\alpha}_{t - 1} \bar{\Sigma})$. The last variable $S_0$ is the clean sample. When compared to \cref{sec:setting}, the prior and evidence are mixed conditionally in a $T$-stage chain. This increases the computational cost $T$ times, as discussed in \cref{sec:conclusions}.

\subsection{Key Approximation in \cref{thm:linear posterior}}
\label{sec:key approximation}

Now we motivate our assumption in \eqref{eq:integral approximation}. Simply put, we assume that $s_0 = s_t / \sqrt{\bar{\alpha}_t}$, where $s_0$ is a clean sample and $s_t$ is the corresponding diffused sample in stage $t$. This is motivated by the forward process, which relates $s_t$ and $s_0$ as $s_t = \sqrt{\bar{\alpha}_t} s_0 + \sqrt{1 - \bar{\alpha}_t} \varepsilon_t$, where $\varepsilon_t \sim \cN(\mathbf{0}_d, I_d)$ is a standard Gaussian noise \citep{ho20denoising}. After rearranging, we get $s_0 = (s_t - \sqrt{1 - \bar{\alpha}_t} \varepsilon_t) / \sqrt{\bar{\alpha}_t}$, and therefore $s_0$ can be viewed as a random variable with mean $s_t / \sqrt{\bar{\alpha}_t}$. The consequence of \eqref{eq:integral approximation} is that the likelihood becomes a function of $s_t$, which yields a closed form when multiplied by the conditional prior, also a function of $s_t$. Our approximation can be also viewed as the Tweedie's formula in \citet{chung23diffusion} where the score component is neglected.

Our approximation has several notable properties. First, $\sqrt{(1 - \bar{\alpha}_t) / \bar{\alpha}_t} \to 0$ as $t \to 1$. Therefore, it becomes more precise in later stages of the reverse process. Second, in the absence of evidence $h$, the approximation vanishes, and all posterior distributions in \cref{thm:linear posterior} reduce to the priors in \eqref{eq:reverse process}. Finally, as the number of observations increases, sampling from the posterior in \cref{thm:linear posterior} is asymptotically consistent.

\begin{theorem}
\label{thm:asymptotic consistency} Fix $\theta_* \in \realset^d$. Let $\tilde{\theta} \gets {\laplacedps}((\mu_t, \Sigma_t)_{t \in [T]}, h)$, where $h = \set{(\phi_\ell, y_\ell)}_{\ell \in [N]}$ is a \emph{history} of $N$ observations. Suppose that $\lambda_d(\bar{\Sigma}^{-1}) \to \infty$ as $N \to \infty$, where $\bar{\Sigma}$ is defined in \eqref{eq:linear posterior}. Then $\prob{\lim_{N \to \infty} \normw{\tilde{\theta} - \theta_*}{2} = 0} = 1$.
\end{theorem}
\begin{proof}
The proof is in \cref{sec:asymptotic consistency proof}. The key idea is that the conditional posteriors in \eqref{eq:posterior T} and \eqref{eq:posterior t - 1} concentrate at a scaled unknown model parameter $\theta_*$ as the number of observations increases, which we formalize as $\lambda_d(\bar{\Sigma}^{-1}) \to \infty$.
\end{proof}

The bound in \cref{thm:asymptotic consistency} can be interpreted as follows. The sampled parameter $\tilde{\theta}$ approaches the true unknown parameter $\theta_*$ as the number of observations $N$ increases. To guarantee that the posterior shrinks uniformly in all directions, we assume that the number of observations in all directions grows linearly with $N$. This is akin to assuming that $\lambda_d(\bar{\Sigma}^{-1}) = \Omega(N)$. This lower bound can be attained in linear models by getting observations according to the D-optimal design \citep{pukelsheim06optimal}.

\subsection{GLM Posterior}
\label{sec:glm posterior}

The Laplace approximation in GLMs (\cref{sec:glm}) naturally generalizes the exact posterior distribution in linear models (\cref{sec:linear model}). We generalize \cref{thm:linear posterior} to GLMs along the same lines.

\begin{theorem}
\label{thm:glm posterior} Let $p$ be a probability measure over the reverse process (\cref{fig:diffusion model}). Suppose that \eqref{eq:integral approximation} holds for all $t \in [T]$. Then $p(s_T \mid h) \propto \cN(s_T; \hat{\mu}_{T + 1}(h), \hat{\Sigma}_{T + 1}(h))$, where
\begin{align*}
  \hat{\mu}_{T + 1}(h)
  = \sqrt{\bar{\alpha}_T} \dot{\theta}_{T + 1}\,, \quad
  \hat{\Sigma}_{T + 1}(h)
  = \bar{\alpha}_T \dot{\Sigma}_{T + 1}\,, \quad
  \dot{\theta}_{T + 1}, \dot{\Sigma}_{T + 1}
  \gets {\irls}(\mathbf{0}_d, I_d / \bar{\alpha}_T, h)\,.
\end{align*}
For $t \in [T]$, we have $p(s_{t - 1} \mid s_t, h) \propto \cN(s_{t - 1}; \hat{\mu}_t(s_t, h), \hat{\Sigma}_t(h))$, where
\begin{align*}
  \hat{\mu}_t(s_t, h)
  = \sqrt{\bar{\alpha}_{t - 1}} \dot{\theta}_t\,, \quad
  \hat{\Sigma}_t(h)
  = \bar{\alpha}_{t - 1} \dot{\Sigma}_t\,, \quad
  \dot{\theta}_t, \dot{\Sigma}_t
  \gets {\irls}(\mu_t(s_t) / \sqrt{\bar{\alpha}_{t - 1}},
  \Sigma_t / \bar{\alpha}_{t - 1}, h)\,.
\end{align*}
\end{theorem}
\begin{proof}
The proof is in \cref{sec:glm posterior proof}. It has four steps. First, we fix stage $t$ and apply approximation \eqref{eq:integral approximation}. Second, we reparameterize the prior, from a function of $s_t$ to a function of $s_t / \sqrt{\bar{\alpha}_t}$. Third, we combine the likelihood with the prior using the Laplace approximation. Finally, we repameterize the posterior, from a function of $s_t / \sqrt{\bar{\alpha}_t}$ to a function of $s_t$.
\end{proof}

Similarly to \cref{thm:linear posterior}, the distributions in \cref{thm:glm posterior} mix evidence with the diffusion model prior. However, this is done implicitly in \irls. The posterior can be sampled from using \laplacedps, where the mean and covariances would be taken from \cref{thm:glm posterior}. Note that \cref{thm:linear posterior} is a special case of \cref{thm:glm posterior} where the mean function $g$ is an identity.

\section{Application to Contextual Bandits}
\label{sec:bandits}

\begin{algorithm}[t]
  \caption{Contextual Thompson sampling.}
  \label{alg:ts}
  \begin{algorithmic}[1]
    \For{round $k = 1, \dots, n$}
      \State Sample $\tilde{\theta}_k \sim p(\cdot \mid h_k)$, where $p(\cdot \mid h_k)$ is the posterior distribution in \eqref{eq:posterior distribution}
      \State Take action $a_k \gets \argmax_{a \in \cA} r(x_k, a; \tilde{\theta}_k)$ and observe reward $y_k$
    \EndFor
  \end{algorithmic}
\end{algorithm}

Now we apply our posterior sampling approximations (\cref{sec:posterior sampling}) to contextual bandits. A \emph{contextual bandit} \citep{langford08epochgreedy,li10contextual} is a classic model for sequential decision making under uncertainty where the agent takes actions conditioned on context. We denote the \emph{action set} by $\cA$ and the \emph{context set} by $\cX$. The \emph{mean reward} for taking action $a \in \cA$ in context $x \in \cX$ is $r(x, a; \theta_*)$, where $r: \cX \times \cA \times \Theta \to \realset$ is a \emph{reward function} and $\theta_* \in \Theta$ is a \emph{model parameter} (\cref{sec:setting}). The agent interacts with the bandit for $n$ rounds indexed by $k \in [n]$. In round $k$, it observes a \emph{context} $x_k \in \cX$, takes an \emph{action} $a_k \in \cA$, and observes its \emph{stochastic reward} $y_k = r(x_k, a_k; \theta_*) + \varepsilon_k$ with independent noise $\varepsilon_k$. We assume that the noise is zero-mean $\sigma^2$-sub-Gaussian for $\sigma > 0$. The objective of the agent is to maximize its cumulative reward in $n$ rounds, or equivalently to minimize its cumulative regret. We define the \emph{$n$-round regret} as
\begin{align}
  \textstyle
  R(n)
  = \sum_{k = 1}^n \E{r(x_k, a_{k, *}; \theta_*) - r(x_k, a_k; \theta_*)}\,,
  \label{eq:regret}
\end{align}
where $a_{k, *} = \argmax_{a \in \cA} r(x_k, a; \theta_*)$ is the optimal action in round $k$.

Arguably the most popular method for solving contextual bandit problems is Thompson sampling \citep{thompson33likelihood,chapelle11empirical,agrawal13thompson}. The key idea in TS is to use the posterior distribution of $\theta_*$ to explore. This is done as follows. In round $k$, the model parameter is drawn from the posterior in \eqref{eq:posterior distribution}, $\tilde{\theta}_k \sim p(\cdot \mid h_k)$, where $h_k$ is the \emph{history} of all interactions up to round $k$. After that, the agent takes the action with the highest mean reward under $\tilde{\theta}_k$. The pseudo-code of this algorithm is given in \cref{alg:ts}.

A \emph{linear bandit} \citep{dani08stochastic} has a linear reward function $r(x, a; \theta_*) = \phi(x, a)\T \theta_*$, where $\phi: \cX \times \cA \to \realset^d$ is a \emph{feature extractor}. The feature extractor can be non-linear in $x$ and $a$. Therefore, linear bandits can be applied to non-linear functions of $x$ and $a$. The feature extractor can be either learned \citep{riquelme18deep} or hand-crafted. We denote the feature vector of the action in round $k$ by $\phi_k = \phi(x_k, a_k)$. Therefore, the \emph{history} of interactions up to round $k$ is $h_k = \set{(\phi_\ell, y_\ell)}_{\ell \in [k - 1]}$. When the prior distribution is a Gaussian, $p(\theta_*) = \cN(\theta_*; \theta_0, \Sigma_0)$, the posterior in round $k$ is a Gaussian in \eqref{eq:linear posterior} for $h = h_k$. When the prior is a diffusion model, we propose sampling from the posterior using
\begin{align}
  \tilde{\theta}_k
  \gets {\laplacedps}((\mu_t, \Sigma_t)_{t \in [T]}, h_k)\,,
  \label{eq:diffusion posterior sampling}
\end{align}
where $\hat{\mu}_t$ and $\hat{\Sigma}_t$ in \laplacedps are computed as in \cref{thm:linear posterior}. We call this algorithm \diffts.

A \emph{generalized linear bandit} \citep{filippi10parametric,jun17scalable,li17provably,kveton20randomized} is an extension of linear bandits to generalized linear models (\cref{sec:glm}). When $p(\theta_*) = \cN(\theta_*; \theta_0, \Sigma_0)$, the Laplace approximation to the posterior is a Gaussian (\cref{sec:glm}). When the prior is a diffusion model, we propose posterior sampling using \eqref{eq:diffusion posterior sampling}, where $\hat{\mu}_t$ and $\hat{\Sigma}_t$ in \laplacedps are computed as in \cref{thm:glm posterior}.

\section{Experiments}
\label{sec:experiments}

\begin{figure}[t]
  \centering
  \includegraphics[width=5.4in]{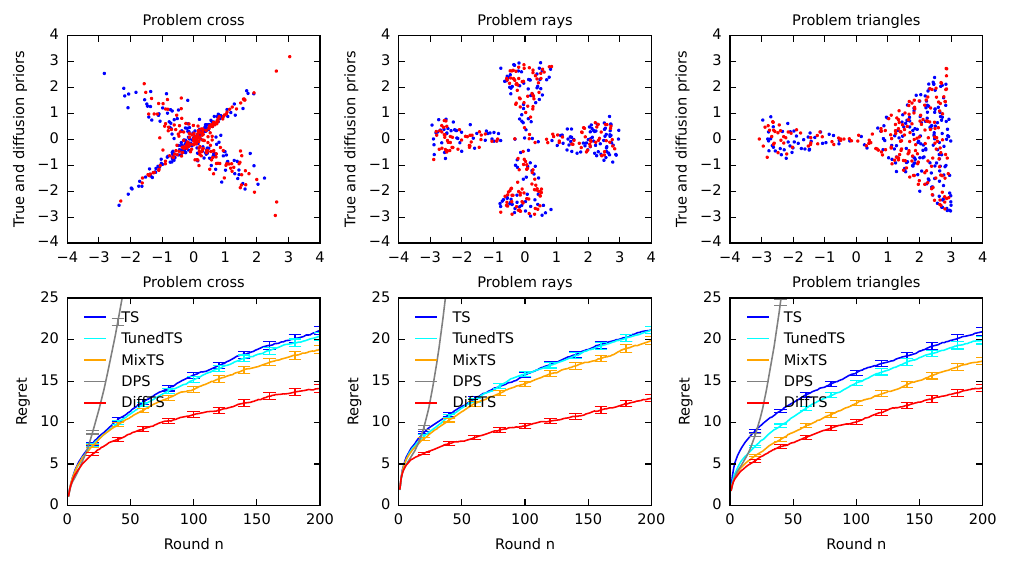}
  \vspace{-0.1in}
  \caption{Evaluation of \diffts on three synthetic problems. The first row shows samples from the true (blue) and diffusion model (red) priors. The second row shows the regret of \diffts and the baselines as a function of round $n$.}
  \label{fig:synthetic experiment}
\end{figure}

We conduct three experiments: synthetic problems in $2$ dimensions (\cref{sec:synthetic experiment,sec:additional synthetic problems}), a recommender system (\cref{sec:movielens experiment}), and a classification problem (\cref{sec:mnist experiment}). In addition, we conduct an ablation study in \cref{sec:ablation studies}, where we vary the number of training samples for the diffusion prior and the number of diffusion stages $T$.

\subsection{Experimental Setup}
\label{sec:experimental setup}

We have four baselines. Three baselines are variants of contextual Thompson sampling \citep{chapelle11empirical,agrawal13thompson}: with an uninformative Gaussian prior (\ts), a learned Gaussian prior (\tunedts), and a learned Gaussian mixture prior (\mixts) \citep{hong22thompson}. The last baseline is diffusion posterior sampling (\dps) of \citet{chung23diffusion}. We implement all TS baselines as described in \cref{sec:bandits}. The uninformative prior is $\cN(\mathbf{0}_d, I_d)$. \mixts is used only in linear bandit experiments because the logistic regression variant does not exist. The TS baselines are chosen to cover various levels of prior information. Our implementation of \dps is described in \cref{sec:dps}. We also experimented with frequentist baselines, such as \linucb \citep{abbasi-yadkori11improved} and the $\varepsilon$-greedy policy. They performed worse than \ts and thus we do not report them.

Each experiment is set up as follows. First, the prior distribution of $\theta_*$ is specified: it can be synthetic or estimated from real-world data. Second, we learn this distribution from $10\,000$ samples from it. In \diffts and \dps, we follow \cref{sec:learning reverse process}. The number of stages is $T = 100$ and the diffusion factor is $\alpha_t = 0.97$. Since $0.97^{100} \approx 0.05$, most of the information in the training samples is diffused. The regressor in \cref{sec:learning reverse process} is a $2$-layer neural network with ReLU activations. In \tunedts, we fit the mean and covariance using maximum likelihood estimation. In \mixts, we fit the Gaussian mixture using \textsc{scikit-learn}. All algorithms are evaluated on $\theta_*$ sampled from the true prior. The regret is computed as defined in \eqref{eq:regret}. All error bars are standard errors of the estimates.

\subsection{Synthetic Experiment}
\label{sec:synthetic experiment}

The first experiment is on three synthetic problems. Each problem is a linear bandit (\cref{sec:bandits}) with $K = 100$ actions in $d = 2$ dimensions. The reward noise is $\sigma = 1$. The feature vectors of actions are sampled uniformly at random from a unit ball. The prior distributions of $\theta_*$ are shown in \cref{fig:synthetic experiment}. The first is a mixture of two Gaussians and the last can be approximated well by a mixture of two Gaussians. We implement \mixts with two mixture components. Therefore, it can represent the first prior exactly and approximate the last one well.

Our results are reported in \cref{fig:synthetic experiment}. We observe two main trends. First, samples from the diffusion prior closely resemble those from the true prior. In such cases, \diffts is expected to perform well and even outperforms \mixts, because it has a better representation of the prior. We observe this in all problems. Second, \dps diverges as the number of rounds increases. This is because \dps uses an approximation based on the likelihood score (\cref{sec:related work}), which is unstable when the score is high. This happens despite our best efforts to tune \dps (\cref{sec:dps}). We report results on another three synthetic problems in \cref{sec:additional synthetic problems}.

\diffts should be $T$ times more computationally costly than \ts with a Gaussian prior (\cref{sec:linear posterior}). We observe this empirically. As an example, the average cost of $100$ runs of \diffts on any problem in \cref{fig:synthetic experiment} is $12$ seconds. The average cost of \ts is $0.1$ seconds. The computation and accuracy can be traded off, and we investigate this in \cref{sec:ablation studies}. In the cross problem, we vary the number of diffusion stages from $T = 1$ to $T = 300$. We observe that the computational cost is linear in $T$ and the regret drops quickly from $26$ at $T = 1$ to $15$ at $T = 50$.

\subsection{MovieLens Experiment}
\label{sec:movielens experiment}

\begin{figure}[t]
  \centering
  \includegraphics[width=5.4in]{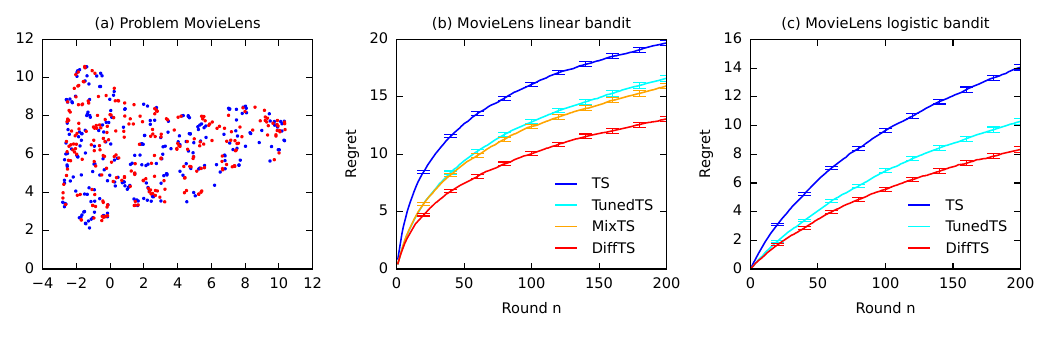}
  \vspace{-0.15in}
  \caption{Evaluation of \diffts on the MovieLens dataset: (a) shows samples from the true (blue) and diffusion model (red) priors, (b) shows regret in the linear bandit, and (c) shows regret in the logistic bandit.}
  \label{fig:movielens experiment}
\end{figure}

In the second experiment, we learn to recommend an item to randomly arriving users. The problem is simulated using the MovieLens 1M dataset \citep{movielens}, with one million ratings for $3\,706$ movies from $6\,040$ users. We subtract the mean rating from all ratings and complete the sparse rating matrix $M$ by alternating least squares \citep{davenport16overview} with rank $d = 5$. The learned factorization is $M = U V\T$. The $i$-th row of $U$, denoted by $U_i$, represents user $i$. The $j$-th row of $V$, denoted by $V_j$, represents movie $j$. We use movie embeddings $V_j$ as model parameters and user embeddings $U_i$ as features of the actions. The movies are items.

We experiment with both linear and logistic bandits. In both, an item is initially chosen randomly from $V_j$ and $K = 10$ actions are chosen randomly from $U_i$ in each round. In the linear bandit, the mean reward of item $j$ for user $i$ is $U_i\T V_j$. The reward noise is $\sigma = 0.75$, and we estimate it from data. In the logistic bandit, the mean reward is $g(U_i\T V_j)$, where $g$ is a sigmoid.

Our MovieLens results are reported in \cref{fig:movielens experiment} and we observe similar trends to \cref{sec:synthetic experiment}. First, samples from the diffusion prior closely resemble those from the true prior (\cref{fig:movielens experiment}a). Since the problem is higher dimensional, we visualize the overlap using UMAP \citep{sainburg21parametric}. Second, \diffts has a lower regret than all baselines, in both linear (\cref{fig:movielens experiment}b) and logistic (\cref{fig:movielens experiment}c) bandits. Finally, \mixts barely outperforms \tunedts. We observe this trend consistently in higher dimensions, and this motivated our work on online learning with more complex priors.

\section{Related Work}
\label{sec:related work}

We start with reviewing related works on bandits with diffusion models. \citet{hsieh23thompson} proposed Thompson sampling with a diffusion model prior for $K$-armed bandits. There are multiple technical differences from our work. First, the diffusion model in \citet{hsieh23thompson} is over scalars representing individual arms. Our model is over vectors representing model parameters, and thus can be applied to contextual bandits. Second, the approximations are different. In stage $t$, \citet{hsieh23thompson} sample from two distributions: the conditional prior and the distribution of the diffused empirical mean up to stage $t$. Then they take a weighted sum of the samples. We sample only once, from the posterior distribution that combines the conditional prior in stage $t$ and likelihood. Therefore, the method of \citet{hsieh23thompson} can be viewed as a non-contextual variant of our method, where posterior sampling is done by weighting samples from the prior and empirical distributions. Finally, \citet{hsieh23thompson} do not analyze their approximation.

\citet{aouali23linear} proposed and analyzed contextual bandits with a linear diffusion model prior: $\mu_t(s_t)$ in \eqref{eq:reverse process} is linear in $s_t$ and $q(s_0)$ is a Gaussian. Therefore, this model is a linear Gaussian model and not a general diffusion model, as in our work.

The closest related work on posterior sampling in diffusion models is \dps of \citet{chung23diffusion}. The key idea in \dps is to sample from the posterior distribution using the likelihood score $\nabla \log p(h \mid \theta)$, where $p(h \mid \theta)$ is the likelihood (\cref{ass:linear model likelihood,ass:glm likelihood}). Note that $\nabla \log p(h \mid \theta)$ grows linearly in $N$ because the history $h$ in $p(h \mid \theta)$ involves $N$ terms. Therefore, \dps becomes unstable as $N \to \infty$. We show it empirically in \cref{sec:synthetic experiment} and discuss the implementation of \dps in \cref{sec:dps}, which was tuned to improve its stability.

Many other posterior sampling methods for diffusion models have been proposed recently: a sequential Monte Carlo approximation for the conditional reverse process \citep{wu23practical}, a variant of \dps with an uninformative prior \citep{meng23diffusion}, a pseudo-inverse approximation to the likelihood of evidence \citep{song23pseudoinverseguided}, and posterior sampling in latent diffusion models \citep{rout23solving}. All of these methods rely on the likelihood score $\nabla \log p(h \mid \theta)$ and thus become unstable as the number of observations $N$ increases. Our posterior approximations do not have this issue because they are based on the product of prior and evidence distributions (\cref{thm:linear posterior,thm:glm posterior}), and thus gradient-free. They work well across different levels of uncertainty (\cref{sec:experiments}) and do not require tuning.

We note that posterior sampling is a special form of inference-time guidance in diffusion models. Other approaches are conditional pre-training \citep{dhariwal21diffusion}, a constraint in the reverse process \citep{graikos22diffusion}, refining the null-space content \citep{wang23zeroshot}, solving an optimization problem that pushes the reverse process towards evidence \citep{song24solving}, and aligning the reverse process with the prompt \citep{bansal24universal}.

\section{Conclusions}
\label{sec:conclusions}

We propose posterior sampling approximations for diffusion models priors. These approximations are contextual, and can be implemented efficiently in linear models and GLMs. We analyze them and evaluate them empirically on contextual bandit problems. Our method has two main limitations.

\textbf{Computational cost.} The cost of posterior sampling in \laplacedps with $T$ stages is about $T$ times higher than that of posterior sampling with a Gaussian prior (\cref{sec:setting}). We validate it empirically in \cref{sec:synthetic experiment}. We plot the sampling time as a function of $T$ in \cref{fig:ablation}c (\cref{sec:ablation studies}).

\textbf{Learning cost and hyper-parameter tuning.} In all experiments, the number of diffusion stages is $T = 100$ and the diffusion rate is set such that most of the signal diffuses. The regressor is a $2$-layer neural network and we learn it from $10\,000$ samples from the prior. These settings resulted in stable performance in all our experiments (\cref{sec:experiments}). However, they clearly impact the performance. We plot the regret as a function of the number of training samples in \cref{fig:ablation}a and as a function of $T$ in \cref{fig:ablation}b. When $T$ or the number of training samples is small, \diffts performs very similarly to posterior sampling with a Gaussian prior. In summary, there is no benefit in these cases.

\textbf{Future work.} We develop novel posterior approximations rather than bounding their regret. This is because the existing approximations are unstable and may diverge in the online setting (\cref{sec:related work,sec:synthetic experiment}). We believe that a proper regret analysis of \diffts is possible and would require bounding two errors. The first error arises because the reverse process does not reverse the forward process exactly (\cref{sec:learning reverse process}). The second error arises because our posterior distributions are approximate (\cref{sec:key approximation}). One possibility is to start with prior works that already showed the utility of complex priors. For instance, \citet{russo16information} proved a $O(\sqrt{\Gamma H(A_*) n})$ regret bound for a linear bandit, where $\Gamma$ is the maximum ratio of regret to information gain and $H(A_*)$ is the entropy of the distribution of the optimal action under the prior. This bound holds for any prior and says that a lower entropy $H(A_*)$, which corresponds to more informative priors, yields a lower regret.

We also believe that our ideas can be extended beyond GLMs. The key idea in \cref{sec:glm posterior} is to use the Laplace approximation of the likelihood. This approximation can be computed exactly in GLMs. More generally though, it is a good approximation whenever the likelihood can be approximated well by a single Gaussian distribution. By the central limit theorem, under appropriate assumptions, this is expected for any observation model when the number of observations is large.

\bibliographystyle{plainnat}
\bibliography{References}

\clearpage
\onecolumn
\appendix

\section{Proofs and Supporting Lemmas}
\label{sec:proofs}

This section contains proofs of our main claims and supporting lemmas.

\subsection{Proof of \cref{lem:chain model posterior}}
\label{sec:chain model posterior proof}

All derivations are based on basic rules of probability and the chain structure in \cref{fig:diffusion model}, and are exact. From \cref{fig:diffusion model}, the joint probability distribution conditioned on $H = h$ factors as
\begin{align*}
  p(s_{0 : T} \mid h)
  = p(s_T \mid h) \prod_{t = 1}^T p(s_{t - 1} \mid s_{t : T}, h)
  = p(s_T \mid h) \prod_{t = 1}^T p(s_{t - 1} \mid s_t, h)\,.
\end{align*}
We use that $p(s_{t - 1} \mid s_{t : T}, h) = p(s_{t - 1} \mid s_t, h)$ in the last equality. We consider two cases.

\noindent \textbf{Derivation of $p(s_{t - 1} \mid s_t, h)$.} By Bayes' rule, we get
\begin{align*}
  p(s_{t - 1} \mid s_t, h)
  = \frac{p(h \mid s_{t - 1}, s_t) \, p(s_{t - 1} \mid s_t)}{p(h \mid s_t)}
  \propto p(h \mid s_{t - 1}) \, p(s_{t - 1} \mid s_t)\,.
\end{align*}
In the last step, we use that $p(h \mid s_t)$ is a constant, since $s_t$ and $h$ are fixed, and that $p(h \mid s_{t - 1}, s_t) = p(h \mid s_{t - 1})$. Note that the last term $p(s_{t - 1} \mid s_t)$ is the conditional prior distribution. When $t > 1$, the first term can be expressed as
\begin{align*}
  p(h \mid s_{t - 1})
  & = \int_{s_0} p(h, s_0 \mid s_{t - 1}) \dif s_0
  = \int_{s_0} p(h \mid s_0, s_{t - 1}) \, p(s_0 \mid s_{t - 1}) \dif s_0 \\
  & = \int_{s_0} p(h \mid s_0) \, p(s_0 \mid s_{t - 1}) \dif s_0\,.
\end{align*}
In the last equality, we use that our graphical model is a chain (\cref{fig:diffusion model}), and thus $p(h \mid s_0, s_{t - 1}) = p(h \mid s_0)$. Finally, we chain all identities and get that
\begin{align}
  p(s_{t - 1} \mid s_t, h)
  \propto \int_{s_0} p(h \mid s_0) \, p(s_0 \mid s_{t - 1}) \dif s_0 \,
  p(s_{t - 1} \mid s_t)\,.
  \label{eq:conditional posterior}
\end{align}

\noindent \textbf{Derivation of $p(s_T \mid h)$.} By Bayes' rule, we get
\begin{align*}
  p(s_T \mid h)
  = \frac{p(h \mid s_T) \, p(s_T)}{p(h)}
  \propto p(h \mid s_T) \, p(s_T)\,.
\end{align*}
In the last step, we use that $p(h)$ is a constant, since $h$ is fixed. The first term can be rewritten as
\begin{align*}
  p(h \mid s_T)
  & = \int_{s_0} p(h, s_0 \mid s_T) \dif s_0
  = \int_{s_0} p(h \mid s_0, s_T) \, p(s_0 \mid s_T) \dif s_0 \\
  & = \int_{s_0} p(h \mid s_0) \, p(s_0 \mid s_T) \dif s_0\,.
\end{align*}
Finally, we chain all identities and get that
\begin{align}
  p(s_T \mid h)
  \propto \int_{s_0} p(h \mid s_0) \, p(s_0 \mid s_T) \dif s_0 \, p(s_T)\,.
  \label{eq:initial posterior}
\end{align}
This completes the derivations.

\subsection{Proof of \cref{thm:linear posterior}}
\label{sec:linear posterior proof}

This proof has two parts.

\noindent \textbf{Derivation of $p(s_{t - 1} \mid s_t, h)$.} From \eqref{eq:integral approximation} and \cref{ass:linear model likelihood}, it follows that
\begin{align*}
  \int_{s_0} p(h \mid s_0) \, p(s_0 \mid s_{t - 1}) \dif s_0
  & \propto p(h \mid s_{t - 1} / \sqrt{\bar{\alpha}_{t - 1}})
  \propto \cN(s_{t - 1} / \sqrt{\bar{\alpha}_{t - 1}}; \bar{\theta}, \bar{\Sigma}) \\
  & \propto \cN(s_{t - 1}; \sqrt{\bar{\alpha}_{t - 1}} \bar{\theta},
  \bar{\alpha}_{t - 1} \bar{\Sigma})\,.
\end{align*}
The last step treats $\bar{\alpha}_{t - 1}$ and $\bar{\Sigma}$ as constants, because the forward process and evidence $h$ are fixed. Now we apply \cref{lem:gaussian product} to distributions
\begin{align*}
  p(s_{t - 1} \mid s_t)
  = \cN(s_{t - 1}; \mu_t(s_t), \Sigma_t)\,, \quad
  \cN(s_{t - 1}; \sqrt{\bar{\alpha}_{t - 1}} \bar{\theta},
  \bar{\alpha}_{t - 1} \bar{\Sigma})\,,
\end{align*}
and get that
\begin{align*}
  p(s_{t - 1} \mid s_t, h)
  \propto \cN(s_{t - 1}; \hat{\mu}_t(s_t, h), \hat{\Sigma}_t(h))\,,
\end{align*}
where $\hat{\mu}_t(s_t, h)$ and $\hat{\Sigma}_t(h)$ are defined in the claim. This is a product of two Gaussians: the prior with mean $\mu_t(s_t)$ and covariance $\Sigma_t$, and the evidence with mean $\sqrt{\bar{\alpha}_{t - 1}} \bar{\theta}$ and covariance $\bar{\alpha}_{t - 1} \bar{\Sigma}$.

\noindent \textbf{Derivation of $p(s_T \mid h)$.} Analogously to the derivation of $p(s_{t - 1} \mid s_t, h)$, we establish that
\begin{align*}
  \int_{s_0} p(h \mid s_0) \, p(s_0 \mid s_T) \dif s_0
  \propto \cN(s_T; \sqrt{\bar{\alpha}_T} \bar{\theta},
  \bar{\alpha}_T \bar{\Sigma})\,.
\end{align*}
Then we apply \cref{lem:gaussian product} to distributions
\begin{align*}
  p(s_T)
  = \cN(s_T; \mathbf{0}_d, I_d)\,, \quad
  \cN(s_T; \sqrt{\bar{\alpha}_T} \bar{\theta},
  \bar{\alpha}_T \bar{\Sigma})\,,
\end{align*}
and get that
\begin{align*}
  p(s_T \mid h)
  \propto \cN(s_T; \hat{\mu}_{T + 1}(h), \hat{\Sigma}_{T + 1}(h))\,,
\end{align*}
where $\hat{\mu}_{T + 1}(h)$ and $\hat{\Sigma}_{T + 1}(h)$ are defined in the claim. This is a product of two Gaussians: the prior with mean $\mathbf{0}_d$ and covariance $I_d$, and the evidence with mean $\sqrt{\bar{\alpha}_T} \bar{\theta}$ and covariance $\bar{\alpha}_T \bar{\Sigma}$.

\subsection{Proof of \cref{thm:asymptotic consistency}}
\label{sec:asymptotic consistency proof}

We start with the triangle inequality
\begin{align*}
  \normw{\tilde{\theta} - \theta_*}{2}
  = \normw{\tilde{\theta} - \bar{\theta} + \bar{\theta} - \theta_*}{2}
  \leq \normw{\tilde{\theta} - \bar{\theta}}{2} + \normw{\bar{\theta} - \theta_*}{2}\,,
\end{align*}
where we introduce $\bar{\theta}$ from \cref{sec:linear model}. Now we bound each term on the right-hand side.

\textbf{Upper bound on $\normw{\tilde{\theta} - \bar{\theta}}{2}$.} This part of the proof is based on analyzing the asymptotic behavior of the conditional densities in \cref{thm:linear posterior}.

As a first step, note that $S_T \sim \cN(\hat{\mu}_{T + 1}(h), \hat{\Sigma}_{T + 1}(h))$, where
\begin{align*}
  \hat{\mu}_{T + 1}(h)
  = \hat{\Sigma}_{T + 1}(h) (I_d \, \mathbf{0}_d +
  \bar{\Sigma}^{-1} \bar{\theta} / \sqrt{\bar{\alpha}_T})\,, \quad
  \hat{\Sigma}_{T + 1}(h)
  = (I_d + \bar{\Sigma}^{-1} / \bar{\alpha}_T)^{-1}\,.
\end{align*}
Since $\lambda_d(\bar{\Sigma}^{-1}) \to \infty$, we get
\begin{align*}
  \hat{\Sigma}_{T + 1}(h)
  \to \bar{\alpha}_T \bar{\Sigma}\,, \quad
  \hat{\mu}_{T + 1}(h)
  \to \sqrt{\bar{\alpha}_T} \bar{\theta}\,.
\end{align*}
Moreover, $\lambda_d(\bar{\Sigma}^{-1}) \to \infty$ implies $\lambda_1(\bar{\Sigma}) \to 0$, and thus $\lim_{N \to \infty} \normw{S_T - \sqrt{\bar{\alpha}_T} \bar{\theta}}{2} = 0$.

The same argument can be applied inductively to later stages of the reverse process. Specifically, for any $t \in [T]$, $S_{t - 1} \sim \cN(\hat{\mu}_t(S_t, h), \hat{\Sigma}_t(h))$, where
\begin{align*}
  \hat{\mu}_t(S_t, h)
  = \hat{\Sigma}_t(h) (\Sigma_t^{-1} \mu_t(S_t) +
  \bar{\Sigma}^{-1} \bar{\theta} / \sqrt{\bar{\alpha}_{t - 1}})\,, \quad
  \hat{\Sigma}_t(h)
  = (\Sigma_t^{-1} + \bar{\Sigma}^{-1} / \bar{\alpha}_{t - 1})^{-1}\,.
\end{align*}
Since $\lambda_d(\bar{\Sigma}^{-1}) \to \infty$ and $S_t \to \sqrt{\bar{\alpha}_t} \bar{\theta}$ by induction, we get
\begin{align*}
  \hat{\Sigma}_t(h)
  \to \bar{\alpha}_{t - 1} \bar{\Sigma}\,, \quad
  \hat{\mu}_t(S_t, h)
  \to \sqrt{\bar{\alpha}_{t - 1}} \bar{\theta}\,.
\end{align*}
Moreover, $\lambda_d(\bar{\Sigma}^{-1}) \to \infty$ implies $\lambda_1(\bar{\Sigma}) \to 0$, and thus $\lim_{N \to \infty} \normw{S_{t - 1} - \sqrt{\bar{\alpha}_{t - 1}} \bar{\theta}}{2} = 0$ for any $t \in [T]$. In the last stage, $t = 1$, $\bar{\alpha}_0 = 1$, and $S_0 = \tilde{\theta}$. Therefore,
\begin{align*}
  \lim_{N \to \infty}
  \normw{S_{t - 1} - \sqrt{\bar{\alpha}_{t - 1}} \bar{\theta}}{2}
  = \lim_{N \to \infty} \normw{\tilde{\theta} - \bar{\theta}}{2}
  = 0\,.
\end{align*}

\textbf{Upper bound on $\normw{\bar{\theta} - \theta_*}{2}$.} This part of the proof uses the definition of $\bar{\theta}$ in \cref{sec:linear model} and that $\varepsilon_\ell \sim \cN(0, \sigma^2)$ is independent noise. By definition,
\begin{align*}
  \bar{\theta} - \theta_*
  = \sigma^{-2} \bar{\Sigma} \sum_{\ell = 1}^N \phi_\ell y_\ell - \theta_*
  = \sigma^{-2} \bar{\Sigma} \sum_{\ell = 1}^N
  \phi_\ell (\phi_\ell\T \theta_* + \varepsilon_\ell) - \theta_*
  = \sigma^{-2} \bar{\Sigma} \sum_{\ell = 1}^N \phi_\ell \varepsilon_\ell\,.
\end{align*}
Since $\varepsilon_\ell$ is independent zero-mean Gaussian noise with variance $\sigma^2$, $\bar{\theta} - \theta_*$ is a Gaussian random variable with mean $\mathbf{0}_d$ and covariance
\begin{align*}
  \cov{\sigma^{-2} \bar{\Sigma} \sum_{\ell = 1}^N \phi_\ell \varepsilon_\ell}
  = \sigma^{-4} \bar{\Sigma} \left(\sum_{\ell = 1}^N \phi_\ell
  \var{\varepsilon_\ell} \phi_\ell\T\right) \bar{\Sigma}
  = \bar{\Sigma} \frac{\sum_{\ell = 1}^N \phi_\ell \phi_\ell\T}{\sigma^2} \bar{\Sigma}
  = \bar{\Sigma}\,.
\end{align*}
Since $\lambda_d(\bar{\Sigma}^{-1}) \to \infty$ implies $\lambda_1(\bar{\Sigma}) \to 0$, we get
\begin{align*}
  \lim_{N \to \infty} \normw{\bar{\theta} - \theta_*}{2}
  = 0\,.
\end{align*}
This completes the proof.

\subsection{Proof of \cref{thm:glm posterior}}
\label{sec:glm posterior proof}

This proof has two parts.

\noindent \textbf{Derivation of $p(s_{t - 1} \mid s_t, h)$.} From \eqref{eq:integral approximation}, we have
\begin{align*}
  \int_{s_0} p(h \mid s_0) \, p(s_0 \mid s_{t - 1}) \dif s_0
  \propto p(h \mid s_{t - 1} / \sqrt{\bar{\alpha}_{t - 1}})\,.
\end{align*}
Since $p(s_{t - 1} \mid s_t)$ is a Gaussian, we have
\begin{align*}
  p(s_{t - 1} \mid s_t)
  = \cN(s_{t - 1}; \mu_t(s_t), \Sigma_t)
  \propto \cN(\gamma s_{t - 1}; \gamma \mu_t(s_t), \gamma^2 \Sigma_t)
\end{align*}
for $\gamma = 1 / \sqrt{\bar{\alpha}_{t - 1}}$. Then by the Laplace approximation,
\begin{align*}
  p(h \mid \gamma s_{t - 1}) \, \cN(\gamma s_{t - 1}; \gamma \mu_t(s_t), \gamma^2 \Sigma_t)
  \propto \cN(\gamma s_{t - 1}; \dot{\theta}_t, \dot{\Sigma}_t)
  \propto \cN(s_{t - 1}; \dot{\theta}_t / \gamma, \dot{\Sigma}_t / \gamma^2)\,,
\end{align*}
where $\dot{\theta}_t, \dot{\Sigma}_t \gets {\irls}(\gamma \mu_t(s_t), \gamma^2 \Sigma_t, h)$.

\noindent \textbf{Derivation of $p(s_T \mid h)$.} Analogously to the derivation of $p(s_{t - 1} \mid s_t, h)$, we establish that
\begin{align*}
  \int_{s_0} p(h \mid s_0) \, p(s_0 \mid s_T) \dif s_0
  \propto p(h \mid s_T / \sqrt{\bar{\alpha}_T})\,.
\end{align*}
Then by the Laplace approximation for $\gamma = 1 / \sqrt{\bar{\alpha}_T}$, we get
\begin{align*}
  p(h \mid \gamma s_T) \, \cN(s_T; \mathbf{0}_d, I_d)
  \propto \cN(s_T; \dot{\theta}_{T + 1} / \gamma, \dot{\Sigma}_{T + 1} / \gamma^2)\,,
\end{align*}
where $\dot{\theta}_{T + 1}, \dot{\Sigma}_{T + 1} \gets {\irls}(\mathbf{0}_d, \gamma^2 I_d, h)$.

\subsection{Supporting Lemmas}
\label{sec:supporting lemmas}

We state and prove our supplementary lemmas next.

\begin{lemma}
\label{lem:multivariate gaussian kl} Let $p(x) = \cN(x; \mu_1, \Sigma_1)$ and $q(x) = \cN(x; \mu_2, \Sigma_2)$, where $\mu_1, \mu_2 \in \realset^d$ and $\Sigma_1, \Sigma_2 \in \realset^{d \times d}$. Then
\begin{align*}
  d(p, q)
  = \frac{1}{2} \left((\mu_2 - \mu_1)\T \Sigma_2^{-1} (\mu_2 - \mu_1) +
  \trace(\Sigma_2^{-1} \Sigma_1) - \log \frac{\det(\Sigma_1)}{\det(\Sigma_2)} - d\right)\,.
\end{align*}
Moreover, when $\Sigma_1 = \Sigma_2$,
\begin{align*}
  d(p, q)
  = \frac{1}{2} (\mu_2 - \mu_1)\T \Sigma_2^{-1} (\mu_2 - \mu_1)\,.
\end{align*}
\end{lemma}
\begin{proof}
The proof follows from the definitions of KL divergence and multivariate Gaussians.
\end{proof}

\begin{lemma}
\label{lem:gaussian product} Fix $\mu_1 \in \realset^d$, $\Sigma_1 \succeq 0$, $\mu_2 \in \realset^d$, and $\Sigma_2 \succeq 0$. Then
\begin{align*}
  \cN(x; \mu_1, \Sigma_1) \, \cN(x; \mu_2, \Sigma_2)
  \propto \cN(x; \mu, \Sigma)\,,
\end{align*}
where
\begin{align*}
  \mu
  = \Sigma (\Sigma_1^{-1} \mu_1 + \Sigma_2^{-1} \mu_2)\,, \quad
  \Sigma
  = (\Sigma_1^{-1} + \Sigma_2^{-1})^{-1}\,.
\end{align*}
\end{lemma}
\begin{proof}
This is a classic result, which is proved as
\begin{align*}
  \cN(x; \mu_1, \Sigma_1) \, \cN(x; \mu_2, \Sigma_2)
  & \propto \exp\left[- \frac{1}{2}
  ((x - \mu_1)\T \Sigma_1^{-1} (x - \mu_1) +
  (x - \mu_2)\T \Sigma_2^{-1} (x - \mu_2))\right] \\
  & \propto \exp\left[- \frac{1}{2}
  (x\T \Sigma_1^{-1} x - 2 x\T \Sigma_1^{-1} \mu_1 +
  x\T \Sigma_2^{-1} x - 2 x\T \Sigma_2^{-1} \mu_2)\right] \\
  & = \exp\left[- \frac{1}{2}
  (x\T \Sigma^{-1} x -
  2 x\T \Sigma^{-1} \Sigma (\Sigma_1^{-1} \mu_1 + \Sigma_2^{-1} \mu_2))\right] \\
  & \propto 
  \exp\left[- \frac{1}{2} (x - \mu)\T \Sigma^{-1} (x - \mu)\right]
  \propto \cN(x; \mu, \Sigma)\,.
\end{align*}
The neglected factors depend on constants $\mu_1$, $\mu_2$, $\Sigma_1$, and $\Sigma_2$. This completes the proof.
\end{proof}

\section{Learning the Reverse Process}
\label{sec:learning reverse process}

One property of our model is that $q(s_T) \approx \cN(s_T; \mathbf{0}_d, I_d)$ when $T$ is sufficiently large \citep{ho20denoising}. Since $S_T$ has the same distribution in the reverse process $p$, $p$ can be learned from the forward process $q$ by simply reversing it. This is done as follows. Using the definition of the forward process in \eqref{eq:forward process}, \citet{ho20denoising} showed that
\begin{align}
  q(s_{t - 1} \mid s_t, s_0)
  = \cN(s_{t - 1}; \tilde{\mu}_t(s_t, s_0), \tilde{\beta}_t I_d)
  \label{eq:stage conditional}
\end{align}
holds for any $s_0$ and $s_t$, where
\begin{align}
  \tilde{\mu}_t(s_t, s_0)
  = \frac{\sqrt{\bar{\alpha}_{t - 1}} \beta_t}{1 - \bar{\alpha}_t} s_0 +
  \frac{\sqrt{\alpha_t} (1 - \bar{\alpha}_{t - 1})}{1 - \bar{\alpha}_t} s_t\,, \quad
  \tilde{\beta}_t
  = \frac{1 - \bar{\alpha}_{t - 1}}{1 - \bar{\alpha}_t} \beta_t\,, \quad
  \bar{\alpha}_t
  = \prod_{\ell = 1}^t \alpha_\ell\,.
  \label{eq:learning reverse process}
\end{align}
Therefore, the latent variable in stage $t - 1$, $S_{t - 1}$, is easy to sample when $s_t$ and $s_0$ are known. To estimate $s_0$, which is unknown when sampling from the reverse process, we use the forward process again. In particular, \eqref{eq:forward process} implies that $s_t = \sqrt{\bar{\alpha}_t} s_0 + \sqrt{1 - \bar{\alpha}_t} \varepsilon_t$, where $\varepsilon_t \sim \cN(\mathbf{0}_d, I_d)$ is a standard Gaussian noise. This identity can be rearranged as
\begin{align*}
  s_0
  = \frac{1}{\sqrt{\bar{\alpha}_t}} (s_t - \sqrt{1 - \bar{\alpha}_t} \varepsilon_t)\,.
\end{align*}
To obtain $\varepsilon_t$, which is unknown when sampling from $p$, we learn to regress it from $s_t$ \citep{ho20denoising}.

The regressor is learned as follows. Let $\varepsilon_t(\cdot; \psi)$ be a regressor of $\varepsilon_t$ parameterized by $\psi$ and $\cD = \set{s_0}$ be a dataset of training examples. We sample $s_0$ uniformly at random from $\cD$ and then solve
\begin{align}
  \psi_t
  = \argmin_\psi \Erv{q}{\normw{\varepsilon_t - \varepsilon_t(S_t; \psi)}{2}^2}
  \label{eq:least squares}
\end{align}
per stage. The expectation is approximated by sampled $s_0$. Note that we slightly depart from \citet{ho20denoising}. Since each regressor has its own parameters, the original optimization problem over $T$ stages decomposes into $T$ subproblems.

\section{Additional Experiments}
\label{sec:additional experiments}

This section contains four additional experiments.

\subsection{Additional Synthetic Problems}
\label{sec:additional synthetic problems}

\begin{figure}[t]
  \centering
  \includegraphics[width=5.4in]{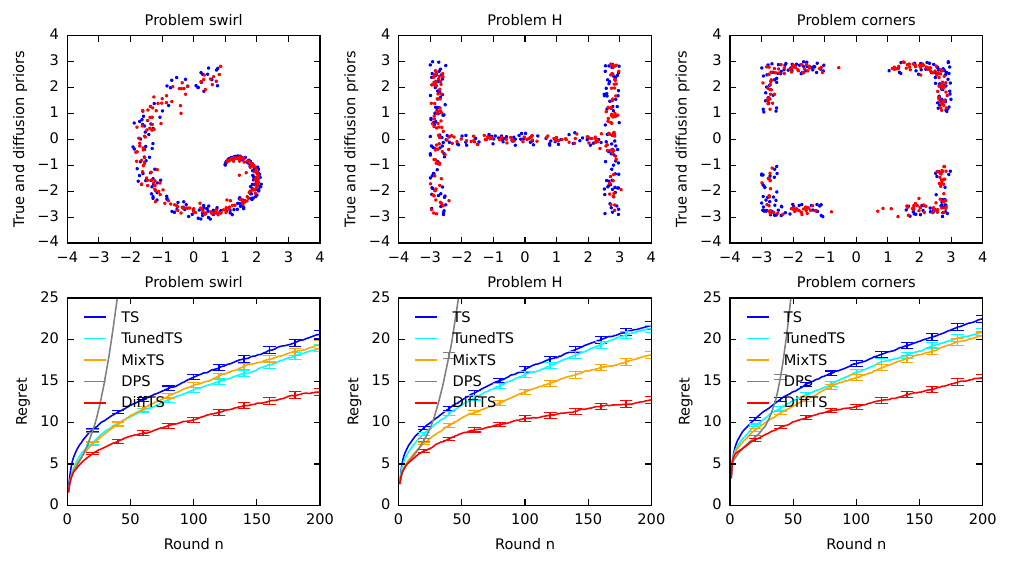}
  \vspace{-0.1in}
  \caption{Evaluation of \diffts on another three synthetic problems. The first row shows samples from the true (blue) and diffusion model (red) priors. The second row shows the regret of \diffts and the baselines as a function of round $n$.}
  \label{fig:synthetic experiment 2}
\end{figure}

In \cref{sec:synthetic experiment}, we show results for three hand-selected problems out of six. We report results on the other three problems in \cref{fig:synthetic experiment 2}. We observe the same trends as in \cref{sec:synthetic experiment}.

\subsection{MNIST Experiment}
\label{sec:mnist experiment}

\begin{figure}[t]
  \centering
  \includegraphics[width=5.4in]{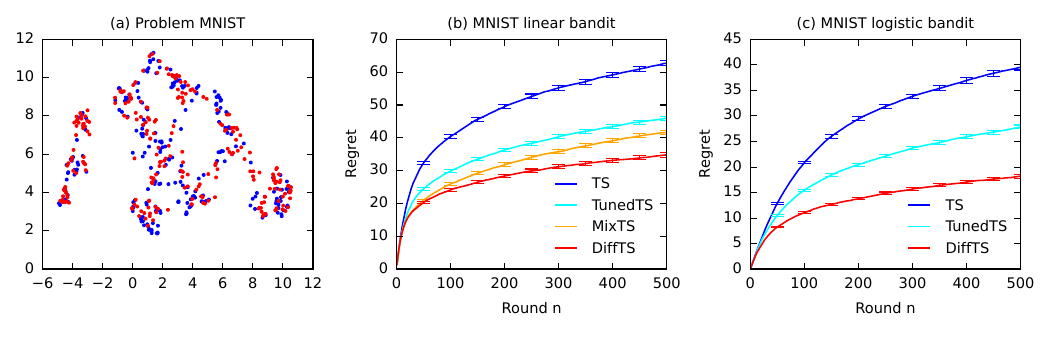}
  \vspace{-0.15in}
  \caption{Evaluation of \diffts on the MNIST dataset: (a) shows samples from the true (blue) and diffusion model (red) priors, (b) shows regret in the linear bandit, and (c) shows regret in the logistic bandit.}
  \label{fig:mnist experiment}
\end{figure}

The next experiment is on the MNIST dataset \citep{mnist}. We start with learning an MLP-based multi-way classifier for digits and extract their $d = 8$ dimensional embeddings. These are used as features in our experiment. We generate a distribution over model parameters $\theta_*$ as follows: (1) we choose a random positive label, assign it reward $1$, and assign reward $-1$ to all other labels; (2) we subsample a random dataset of size $20$, with $50\%$ positive and $50\%$ negative labels; (3) we train a linear model, which gives us a single $\theta_*$. We repeat this $10\,000$ times and get a distribution over $\theta_*$.

We consider both linear and logistic bandits. In both, the model parameter $\theta_*$ is initially sampled from the prior. In each round, $K = 10$ random actions are chosen randomly from all digits. In the linear bandit, the mean reward for a digit with embedding $x$ is $x\T \theta_*$ and the reward noise is $\sigma = 1$. In the logistic bandit, the mean reward is $g(x\T \theta_*)$, where $g$ is a sigmoid.

Our MNIST results are reported in \cref{fig:mnist experiment}. We observe again that \diffts has a lower regret than all baselines, because the learned prior captures the underlying distribution of $\theta_*$ well. We note that both the prior and diffusion prior distributions exhibit a strong cluster structure (\cref{fig:mnist experiment}a), where each cluster represents one label.

\subsection{Ablation Studies}
\label{sec:ablation studies}

\begin{figure}[t]
  \centering
  \includegraphics[width=\textwidth]{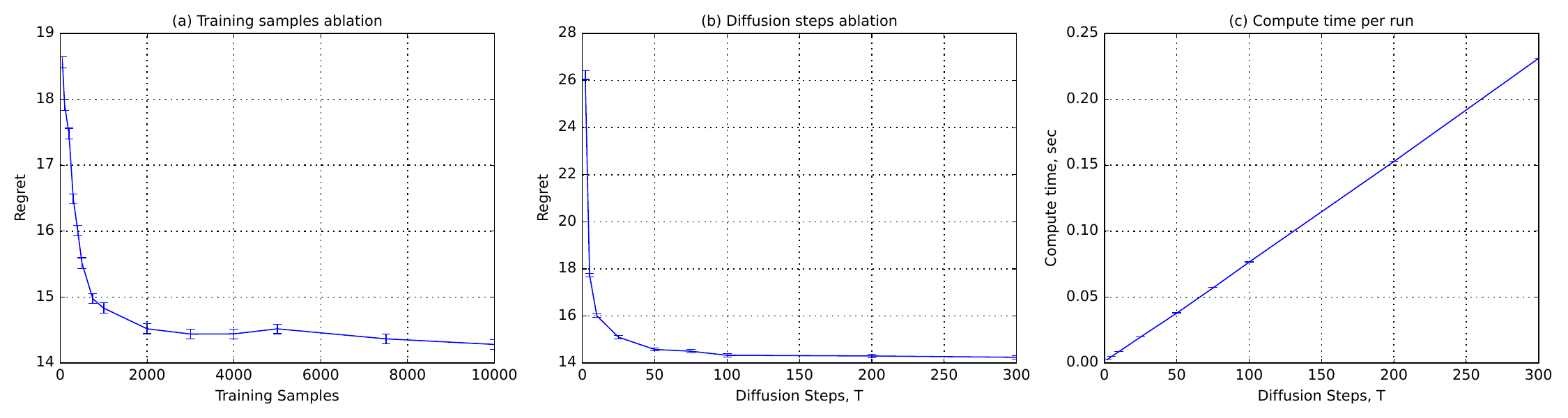}
  \vspace{-0.2in}
  \caption{An ablation study of \diffts on the cross problem: (a) we vary the number of samples for training the diffusion prior and report regret, (b) we vary the number of diffusion stages $T$ and report regret, and (c) we vary the number of diffusion stages $T$ and report computation time.}
  \label{fig:ablation}
\end{figure}

\begin{figure}[t]
  \centering
  \includegraphics[width=5.4in]{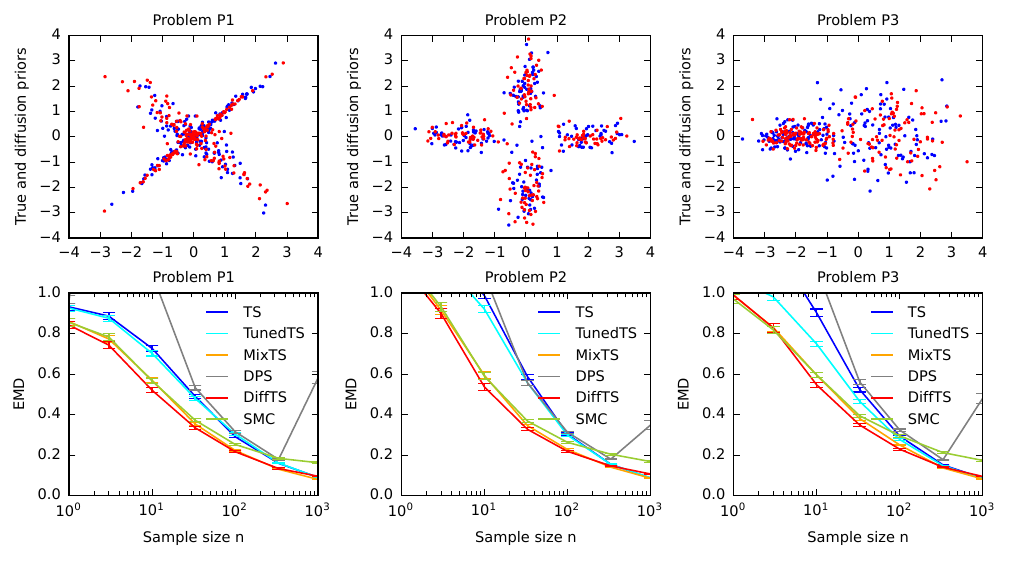}
  \vspace{-0.1in}
  \caption{Evaluation on Gaussian mixture variants of the synthetic problems in Figure 2. The first row shows samples from the true (blue) and diffusion model (red) priors. The second row shows the earth mover's distance of \diffts and baseline posteriors from the true posterior as a function of sample size $n$.}
  \label{fig:posterior quality}
\end{figure}

We conduct three ablation studies on the cross problem in \cref{fig:synthetic experiment}.

In all experiments, the number of samples for training diffusion priors was $10\,000$. In \cref{fig:ablation}a, we vary it from $100$ to $10\,000$. We observe that the regret decreases as the number of samples increases, due to learning a better prior approximation. The trend stabilizes around $3\,000$ training samples. We conclude that the quality of the learned prior approximation has a major impact on regret.

In all experiments, the number of diffusion stages was $T = 100$. In \cref{fig:ablation}b, we vary it from $1$ to $300$ and observe its impact on regret. While the regret at $T = 1$ is high, it decreases quickly as $T$ increases. It stabilizes around $T = 100$, which we used in our experiments. In \cref{fig:ablation}c, we vary $T$ from $1$ to $300$ and observe its effect on the computation time of posterior sampling. The time is linear in $T$, as suggested in \cref{sec:linear posterior}. The main contributor is the neural network regressor.

\subsection{Non-Bandit Evaluation}
\label{sec:non-bandit evaluation}

We use Gaussian mixture variants of the synthetic problems in \cref{fig:synthetic experiment} for our non-bandit evaluation. The action in round $k$ is chosen uniformly at random (not adaptively). Since the priors are Gaussian mixtures, the true posterior distribution can be computed in a closed form using \mixts and we can measure the distance of posterior approximations from it. We use the \emph{earth mover's distance (EMD)} between posterior samples from the true posterior and its approximation. We also considered the KL divergence. However, we could not apply it because the posteriors of \diffts and \dps do not have analytical forms.

We evaluate all methods from \cref{fig:synthetic experiment}. In addition, we implement a \emph{sequential Monte Carlo (SMC)} sampler \citep{doucet01sequential}. The initial particles are sampled uniformly at random from the prior. At each round, the particles are perturbed by a Gaussian noise. The standard deviation of the noise is initialized as a fraction of the observation noise and decays over time, as the posterior concentrates. The particles are weighted according to the likelihood of the observation in the round. Finally, we use normalized likelihood weights to resample the particles. The number of particles is $3\,000$ and we tune SMC to get good posterior approximations. The computational cost of SMC is comparable to \diffts.

Our results are reported in \cref{fig:posterior quality}. We observe that \diffts approximations are comparable to \mixts, which has an exact posterior in this setting. The second best performing method is SMC. Its approximations worsen as the sample size $n$ increases. \dps approximations also get worse as $n$ increases, which caused instability in \cref{fig:synthetic experiment}.

\section{Implementation of \citet{chung23diffusion}}
\label{sec:dps}

\begin{algorithm}[t]
  \caption{\dps of \citet{chung23diffusion}.}
  \label{alg:dps}
  \begin{algorithmic}[1]
    \State \textbf{Input:} Model parameters $\tilde{\sigma}_t$ and $\zeta_t$
    \Statex \vspace{-0.05in}
    \State Initial sample $S_T \sim \cN(\mathbf{0}_d, I_d)$
    \For{stage $t = T, \dots, 1$}
      \State $\hat{S} \gets - \frac{\varepsilon_t(S_t; \psi_t)}{1 - \bar{\alpha}_t}$
      \State $\hat{S}_0 \gets \frac{1}{\sqrt{\bar{\alpha}_t}}
      (S_t + (1 - \bar{\alpha}_t) \hat{S})$
      \State $Z \sim \cN(\mathbf{0}_d, I_d)$
      \State $S_{t - 1} \gets
      \frac{\sqrt{\bar{\alpha}_{t - 1}} \beta_t}{1 - \bar{\alpha}_t} \hat{S}_0 +
      \frac{\sqrt{\alpha_t} (1 - \bar{\alpha}_{t - 1})}{1 - \bar{\alpha}_t} S_t +
      \tilde{\sigma}_t Z -
      \zeta_t \nabla \sum_{\ell = 1}^N (y_\ell - \phi_\ell\T \hat{S}_0)^2$
    \EndFor
    \Statex \vspace{-0.1in}
    \State \textbf{Output:} Posterior sample $S_0$
  \end{algorithmic}
\end{algorithm}

In our experiments, we compare to diffusion posterior sampling (\dps) with a Gaussian observation noise (Algorithm 1 in \citet{chung23diffusion}). Our implementation is presented in \cref{alg:dps}. The score is $\hat{S} = - \varepsilon_t(S_t; \psi_t) / (1 - \bar{\alpha}_t)$, where $\varepsilon_t(S_t; \psi_t)$ is a regression estimate of the forward process noise $\varepsilon_t$ in \cref{sec:learning reverse process}. We set $\tilde{\sigma}_t = \sqrt{\tilde{\beta}_t}$, which is the same amount of noise as in our reverse process (\cref{sec:diffusion models}). The term
\begin{align*}
  \nabla \sum_{\ell = 1}^N (y_\ell - \phi_\ell\T \hat{S}_0)^2
\end{align*}
is the gradient of the negative log-likelihood with respect to $S_t$.

As discussed in Appendices C.2 and D.1 of \citet{chung23diffusion}, $\zeta_t$ in \dps needs to be tuned for good performance. This is because $\nabla \sum_{\ell = 1}^N (y_\ell - \phi_\ell\T \hat{S}_0)^2$ grows with the number of observations, which causes instability. We also observed this in our experiments (\cref{sec:synthetic experiment}). To make \dps work well, we follow \citet{chung23diffusion} and set
\begin{align*}
  \zeta_t
  = \frac{1}{\sqrt{\sum_{\ell = 1}^N (y_\ell - \phi_\ell\T \hat{S}_0)^2}}\,.
\end{align*}
While this significantly improves the performance of \dps, it does not prevent failures. The fundamental problem is that gradient-based optimization is sensitive to the step size, especially when the optimized function is steep. Note that \laplacedps does not have any such hyper-parameter.

\clearpage

\section*{NeurIPS Paper Checklist}

\begin{enumerate}

\item {\bf Claims}
    \item[] Question: Do the main claims made in the abstract and introduction accurately reflect the paper's contributions and scope?
    \item[] Answer: \answerYes{}
    \item[] Justification: The abstract and introduction clearly state all contributions. The introduction also points to where those contributions are made.

\item {\bf Limitations}
    \item[] Question: Does the paper discuss the limitations of the work performed by the authors?
    \item[] Answer: \answerYes{}
    \item[] Justification: The increase in computational cost is discussed in \cref{sec:linear posterior} and shown empirically in \cref{sec:synthetic experiment}. We also conduct an ablation study in \cref{sec:ablation studies}, where we show how the regret of \diffts scales with the number of samples used for pre-training the prior and the number of stages in the diffusion model prior.

\item {\bf Theory Assumptions and Proofs}
    \item[] Question: For each theoretical result, does the paper provide the full set of assumptions and a complete (and correct) proof?
    \item[] Answer: \answerYes{}
    \item[] Justification: The main claims are stated and discussed in \cref{sec:posterior sampling}. Their proofs are in \cref{sec:proofs}.

\item {\bf Experimental Result Reproducibility}
    \item[] Question: Does the paper fully disclose all the information needed to reproduce the main experimental results of the paper to the extent that it affects the main claims and/or conclusions of the paper (regardless of whether the code and data are provided or not)?
    \item[] Answer: \answerYes{}
    \item[] Justification: We also include code to reproduce the synthetic results in \cref{fig:synthetic experiment,fig:synthetic experiment 2}.

\item {\bf Open access to data and code}
    \item[] Question: Does the paper provide open access to the data and code, with sufficient instructions to faithfully reproduce the main experimental results, as described in supplemental material?
    \item[] Answer: \answerYes{}
    \item[] Justification: We include code to reproduce the synthetic results in \cref{fig:synthetic experiment,fig:synthetic experiment 2}.

\item {\bf Experimental Setting/Details}
    \item[] Question: Does the paper specify all the training and test details (e.g., data splits, hyperparameters, how they were chosen, type of optimizer, etc.) necessary to understand the results?
    \item[] Answer: \answerYes{}
    \item[] Justification: The experiments are described to a sufficient level to be reproducible. To make sure, we include code to reproduce the synthetic results in \cref{fig:synthetic experiment,fig:synthetic experiment 2}.

\item {\bf Experiment Statistical Significance}
    \item[] Question: Does the paper report error bars suitably and correctly defined or other appropriate information about the statistical significance of the experiments?
    \item[] Answer: \answerYes{}
    \item[] Justification: All plots in the paper have error bars.

\item {\bf Experiments Compute Resources}
    \item[] Question: For each experiment, does the paper provide sufficient information on the computer resources (type of compute workers, memory, time of execution) needed to reproduce the experiments?
    \item[] Answer: \answerNo{}
    \item[] Justification: Our experiments are not large scale.
    
\item {\bf Code Of Ethics}
    \item[] Question: Does the research conducted in the paper conform, in every respect, with the NeurIPS Code of Ethics \url{https://neurips.cc/public/EthicsGuidelines}?
    \item[] Answer: \answerYes{}
    \item[] Justification: We checked the link and comply.

\item {\bf Broader Impacts}
    \item[] Question: Does the paper discuss both potential positive societal impacts and negative societal impacts of the work performed?
    \item[] Answer: \answerNA{}
    \item[] Justification: This work is algorithmic and not tied to a particular application that would have immediate negative impact.
    
\item {\bf Safeguards}
    \item[] Question: Does the paper describe safeguards that have been put in place for responsible release of data or models that have a high risk for misuse (e.g., pretrained language models, image generators, or scraped datasets)?
    \item[] Answer: \answerNA{}
    \item[] Justification: This paper does not pose such a risk.

\item {\bf Licenses for existing assets}
    \item[] Question: Are the creators or original owners of assets (e.g., code, data, models), used in the paper, properly credited and are the license and terms of use explicitly mentioned and properly respected?
    \item[] Answer: \answerYes{}
    \item[] Justification: All used assets are stated and cited.

\item {\bf New Assets}
    \item[] Question: Are new assets introduced in the paper well documented and is the documentation provided alongside the assets?
    \item[] Answer: \answerNA{}
    \item[] Justification: This paper does not release new assets.

\item {\bf Crowdsourcing and Research with Human Subjects}
    \item[] Question: For crowdsourcing experiments and research with human subjects, does the paper include the full text of instructions given to participants and screenshots, if applicable, as well as details about compensation (if any)? 
    \item[] Answer: \answerNA{}
    \item[] Justification: No crowdsourcing or research with human subjects.

\item {\bf Institutional Review Board (IRB) Approvals or Equivalent for Research with Human Subjects}
    \item[] Question: Does the paper describe potential risks incurred by study participants, whether such risks were disclosed to the subjects, and whether Institutional Review Board (IRB) approvals (or an equivalent approval/review based on the requirements of your country or institution) were obtained?
    \item[] Answer: \answerNA{}
    \item[] Justification: No crowdsourcing or research with human subjects.

\end{enumerate}

\end{document}